\documentclass[conference]{IEEEtran}
\IEEEoverridecommandlockouts

\usepackage{cite}
\usepackage{amsmath}
\usepackage{amssymb}
\usepackage{amsthm}
\usepackage{import}
\usepackage{tikz}
\usetikzlibrary{arrows,automata,positioning,calc,shapes.multipart,chains,babel,fit,backgrounds}
\def\BibTeX{{\rm B\kern-.05em{\sc i\kern-.025em b}\kern-.08em
    T\kern-.1667em\lower.7ex\hbox{E}\kern-.125emX}}

\tikzset{%
    block/.style={draw, fill=white, rectangle, 
            minimum height=2em, minimum width=2em},
    input/.style={inner sep=0pt},       
    output/.style={inner sep=0pt},      
    sum/.style = {draw, fill=white, circle, minimum size=2mm, node distance=1.5cm, inner sep=0pt},
    pinstyle/.style = {pin edge={to-,thin,black}}
}
\tikzstyle{block} = [draw, rectangle, text width=1.5cm, text centered, minimum height=1cm, node distance=3.3cm,fill=white]
\tikzstyle{container} = [draw, rectangle, inner sep=0.8cm, fill=white,minimum height=2cm]
\def\bottom#1#2{\hbox{\vbox to #1{\vfill\hbox{#2}}}}
\tikzset{
  mybackground/.style={execute at end picture={
      \begin{scope}[on background layer]
        \node[] at (current bounding box.south){\bottom{1cm} #1};
        \end{scope}
    }},
}

\newtheorem{thm}{Theorem}

\theoremstyle{definition}
\newtheorem{definition}{Definition}[section]
\theoremstyle{definition}

\newcommand{\LU}{L_{\text{underfit}}}
\DeclareMathOperator{\E}{\mathbb{E}}
\DeclareMathOperator{\D}{\mathcal{D}}

\let\OldTitle\title
\renewcommand{\title}[1]{\OldTitle{\resizebox{1.\linewidth}{!}{#1}}}

\title{Undecidability of Underfitting in Learning Algorithms}

\author{\IEEEauthorblockN{Sonia Sehra}
\IEEEauthorblockA{Cloud and AI \\
\textit{Microsoft}\\
Redmond, WA, USA \\
sonia.sehra@microsoft.com}
\and
\IEEEauthorblockN{David Flores}
\IEEEauthorblockA{AMISTAD Lab \\
\textit{Harvey Mudd College}\\
Claremont, CA, USA \\
deflores@hmc.edu}
\and
\IEEEauthorblockN{George D.\ Monta\~{n}ez}
\IEEEauthorblockA{
AMISTAD Lab \\
\textit{Harvey Mudd College}\\
Claremont, CA, USA \\
gmontanez@hmc.edu}
}

\begin{document}

\maketitle
\thispagestyle{empty}
\pagestyle{empty}

\begin{abstract}
Using recent machine learning results that present an information-theoretic perspective on underfitting and overfitting, we prove that deciding whether an encodable learning algorithm will always underfit a dataset, even if given unlimited training time, is undecidable.
We discuss the importance of this result and potential topics for further research, including information-theoretic and probabilistic strategies for bounding learning algorithm fit. \end{abstract}

\section{Introduction}

Overfitting and underfitting are often explained as symptoms of the bias-variance trade-off \cite{geman1992neural,kohavi1996bias}, where overfitting describes when a learning method has low training but high test error, and underfitting occurs when a method has high training and high test error. Seeking robust definitions of overfitting and underfitting, we build on recent work in machine learning looking at both phenomena from an information-theoretic perspective \cite{bashir-2020-ITP}.

Information-theoretic notions of overfitting  reflect when algorithms go beyond learning true regularities in data, inadvertently capturing noise in their modeling processes. In contrast, underfitting describes when an algorithm ``fails to capture enough'' information to learn the true regularities in data. Marrying this intuitive notion of underfitting with existing definitions of underfitting can help us derive new insights and define specific circumstances into when and how algorithms underfit.

In this paper, we examine definitions of underfitting using notions of algorithmic capacity and dataset complexity developed by Bashir et al. \cite{bashir-2020-ITP}, and further, prove that determining whether an arbitrary learning algorithm will always \textit{underfit} a particular dataset is formally undecidable.
We will demonstrate this by a reduction from the halting problem.

Our proof for the undecidability of the underfitting uses definitions from Bashir et al. \cite{bashir-2020-ITP}, building on that work while focusing primarily on underfitting. In their aforementioned paper, Bashir et al.\ proved the formal undecidability of the overfitting problem, while leaving open the problem of establishing a similar proof for underfitting. Inspired by their method, we accomplish that task here.

\section{Related Work}

This paper builds on the larger algorithmic search framework for machine learning
developed by Monta{\~n}ez~\cite{montanez2017machine}, on which Bashir et al.\ also build~\cite{bashir-2020-ITP}. Our notions of algorithm and
dataset complexity tie into an existing body of work developed by Lauw et al.~\cite{lauw2020bias}, while using a definition of underfitting most closely related to that of Bashir et al.~\cite{bashir-2020-ITP}. These
papers build on previous work characterizing the capacity of machine learning
algorithms, such as the VC dimension \cite{vapnik2015uniform}, Rademacher
complexity \cite{rademacherlecture}, and Labeling Distribution Matrices
\cite{pss2020LDM}.

There has been a large body of research into underfitting in machine learning
algorithms, and approaches that can be taken to detect and avoid underfitting
algorithms. Gavrilov et al.\ studied causes of over- and underfitting, and
methods to prevent it, in Convolutional Neural Nets
\cite{gavrilov2018preventing}. Li et al.\ did similar work for decision trees
\cite{li2016solving}, and Narayan et al.\ did similar work for Multilayer
Perceptrons \cite{narayan2005analysis}. Many other examples could be given.

There has also been other recent work on decidability for
problems related to machine learning. Ben-David et al.\ recently proved that the
learnability problem is undecidable \cite{ben2019learnability}. Building on
that work, Gandolfi proved that the sample complexity of
an algorithm, a measure of how many samples are needed to solve a problem, is
decidable in some circumstances \cite{gandolfi2020decidability}.

\subsection{Relation to Bashir et al. \cite{bashir-2020-ITP}}   

As we have already noted, this paper is largely built around the definitions developed by Bashir et al. \cite{bashir-2020-ITP}.
 Our work accomplishes three main goals: (1) it expands on the brief description of underfitting presented in that paper, (2) gives a new model-specific definition of underfitting, and (3) proves the formal undecidability of underfitting under this definition.

While Bashir et al.\
introduced an information-theoretic definition for underfitting, their analysis primarily focused on overfitting rather than underfitting.
Our paper analyzes their definition of underfitting, creates a complementary model-specific definition of underfitting, and finally presents rigorous conclusions which can be drawn from that
definition. In particular, we prove that underfitting, like overfitting, is 
formally undecidable.

\section{Background}
For completeness, we reproduce several of the definitions developed by Bashir et al., and employ the same assumptions regarding algorithms, datasets, and hypothesis spaces that are used in that work.

We begin with Definition 4 from Bashir et al., a definition of time-indexed capacity. Let $\mathcal{G}$ denote a finite hypothesis space available to a learning algorithm $\mathcal{A}$ (or more generally, let $\mathcal{G}$ be a finite search space sampled by a search algorithm $\mathcal{A}$).
\begin{definition}[Time-indexed Capacity, from~\cite{bashir-2020-ITP}]\label{def:TIME-INDEXED-CAPACITY}
Let $P_i$ denote the (stochastic) probability distribution over $\mathcal{G}$ at time $i$. $\mathcal{A}$'s capacity \textit{at time i} is the maximum amount of information $\mathcal{A}$ may transfer from a dataset $D \sim \D$ to $G_i \sim \E[P_i|D]$,
    $$C^i_{\mathcal{A}} = \sup_{\D} I(G_i;D),$$
where $\mathcal{A}$ is a learning algorithm and $\mathcal{G}$ is the algorithm's hypothesis space.
\end{definition}
Definition~\ref{def:TIME-INDEXED-CAPACITY} gives us a limit on how much information a learning algorithm is able to extract from a dataset at time $i$ in reference to a finite hypothesis space $\mathcal{G}$. It does so by making use of the information-theoretic quantity of mutual information, which captures the level of dependence between two random variables.

Next, we present Definitions 6 and 7 from Bashir et al., for Dataset Turing Complexity and Dataset Complexity. These definitions allow us to characterize the (algorithmic) information content of arbitrary datasets.
\begin{definition}[Dataset Turing Complexity, from~\cite{bashir-2020-ITP}]
Given a fixed Turing Machine $M$ that accepts a string $p$ and feature vector $x$ as input and outputs a label $y$, the \textit{data complexity} of a dataset $D$ is
$$C_{D,M} = L(\langle M \rangle) + L(p),$$
where $L(p) = \min \{|p|: \forall (\textbf{x},y) \in D, M(p,\textbf{x}) = y\}$. That is, the data complexity $C_{D,M}$ is the length of the shortest program that can correctly map every input in the dataset $\mathcal{D}$ to its corresponding output.
\end{definition}

Before we present the definition of Dataset Complexity from Bashir et al., we need to introduce the quantity $C'_D$.
\begin{definition}[$C'_D$, from \cite{bashir-2020-ITP}]\label{def:UNCOMPRESSED-COMPLEXITY}
$$C'_D = \sum_{i=1}^n b(z_i),$$
where $b(z_i)$ is the number of bits required to encode the feature-label pair $z_i$ from dataset $D=(z_1, \dots, z_n)$, without any compression.
\end{definition}
As noted in Bashir et al., $C'_D$ represents the amount of information needed to memorize a dataset $D$ without compression.
Now, we define Dataset Complexity, following Definition 7 of Bashir et al.:
\begin{definition}[Dataset Complexity, from Bashir et al.~\cite{bashir-2020-ITP}]
$C_D = \min\{C_{D,M}, C'_D\}$.
\end{definition}
Note that $C_D \geq C'_D$, allowing the dataset complexity $C_D$ to be easily upper bounded, which is helpful since $C_{D,M}$ is generally uncomputable.

Having provided some background definitions from Bashir et al., we define underfitting and present several related definitions.
Note, we assume that $D \sim \mathcal{D}$ whenever the distribution of $D$ is not stated explicitly in the definitions that follow.

First, Definition 9 of Bashir et al.\ defines underfitting at iteration $i$ as:

\begin{definition}[Underfitting, from Bashir et al.~\cite{bashir-2020-ITP}]\label{def:underfit}
An algorithm $\mathcal{A}$ \textbf{underfits} at iteration $i$ if $$C^i_{\mathcal{A}} < \E_{\mathcal{D}}[C_D]$$
i.e., after training for $i$ timesteps, $\mathcal{A}$ has time-indexed capacity strictly less than $\E_{\mathcal{D}}[C_D]$.
\end{definition}
As noted in \cite{bashir-2020-ITP}, underfitting \textit{``could be the result of insufficient capacity, insufficient training, or insufficient information retention, all of which are captured by $C^i_{\mathcal{A}}$.''}

The pointwise information transfer from a single dataset to a single model is also defined, which is useful in defining overfitting and underfitting of a particular model (hypothesis) $g$ on a particular dataset $d$.
\begin{definition}[Pointwise Information Transfer \cite{bashir-2020-ITP}]
\label{defn:pointwise-transfer}
For a given dataset $d$ and specific hypothesis $g$, the \textbf{pointwise information transfer} by algorithm $\mathcal{A}$ from $d$ to $g$ is the pointwise mutual information (lift),
\begin{align*}
    C_{\mathcal{A}}(g,d) &= \log_2 \frac{p(g,d)}{p(g)p(d)} = \log_2 \frac{p(g|d)}{p(g)} = \log_2 \frac{p(d|g)}{p(d)}. 
\end{align*}
\end{definition}

In the same way that Bashir et al.\ considered overfitting for a single model and dataset, using pointwise information transfer we can define model underfitting. 
\begin{definition}[Model Underfit]\label{def:MODEL-UNDERFIT}
$\mathcal{A}$'s \textbf{model $g$ underfits $d$} if $C_{\mathcal{A}}(g,d) < C_d$.
\end{definition}

These definitions will allow us to prove the formal undecidability of the underfitting problem, which we turn to next.

\section{Undecidability of Underfitting}\label{sec:UNDECIDABILITY}
We now demonstrate the undecidability of underfitting as Bashir et al.\ did for overfitting.
Whether an iteratively trained learning algorithm underfits, fits, or overfits is often a matter of just how much training has been allowed. Some algorithms initially produce models that underfit, but with more training eventually fit the data well. Other learning methods suffer from such low representational capacity that they can never produce a model that fits, even with unlimited training time. We prove here that determining whether an algorithm will eventually produce a model that does not underfit is formally undecidable. Like Bashir and collaborators, we do so by a reduction from the halting problem.

While we derive our proof using the model-specific notion of underfitting from Definition~\ref{def:MODEL-UNDERFIT}, without too much work the proof can be modified to employ the expectation-centric definition of underfitting from Definition~\ref{def:underfit}. This suggests that even determining if an algorithm will always underfit relative to a distribution on datasets is also formally undecidable, since we can always create an algorithm whose maximum capacity changes whenever a Turing machine $M$ halts on an input $w$. Thus, it doesn't matter if we're comparing the pointwise algorithm capacity against a fixed dataset complexity or the maximum algorithm capacity against an expected dataset complexity; it is the changing algorithm capacity in response to algorithmic halting that does the work in the proof. We present our main result next.

\begin{thm}[The Undecidability of Underfitting]
\label{thm:UNDECIDABLE-UNDERFITTING}
    Let $S$ be the set of all encodable learning algorithms and let $\langle \mathcal{A}\rangle$ denote the encoded form of algorithm $\mathcal{A}$. Then, for any dataset $d$,
    \[
        \LU = \{\langle \mathcal{A}\rangle, d | \mathcal{A} \in S, \mathcal{A} \text{ underfits } d \text{ at all iterations}\}
    \]
    is undecidable.
\end{thm}

\begin{proof}
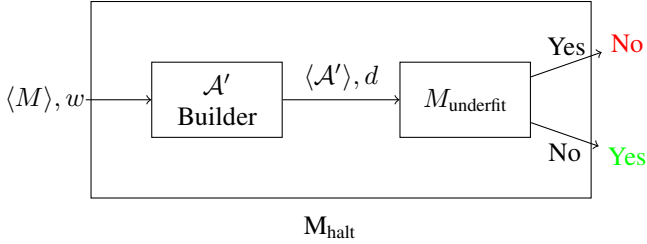
\begin{figure}[htbp]
    \centering
    \begin{tikzpicture}[mybackground={M$_{\text{halt}}$}]
        \node [input, name=text1] {$\langle M \rangle, w$};
        \node [block, right=0.87cm of text1] (text2) {$\mathcal{A}'$\\Builder};
        \node [block, right of=text2] (text3) {$M_{\text{underfit}}$};
        
        \node [above right=0cm and 0.94cm of text3, node distance=5cm] (text4) {\color{red} No};
        \node [below of=text4, node distance=1.5cm] (text5) {\color{green} Yes};
    
        \begin{scope}[on background layer]
            \node [container,fit=(text2) (text3)] (container) {};
        \end{scope}
        
        \draw [->] (text1) -- (text2);
        \draw [->] (text2) -- node [text width=2cm,midway,above,align=center] {$\langle\mathcal{A}'\rangle, d$} (text3);
        \draw [->] (text3) -- node [above] {Yes} (text4);
        \draw [->] (text3) -- node [below] {No} (text5);
    \end{tikzpicture}
    \caption{$M_{\text{halt}}$ constructed using $M_{\text{underfit}}$.}
    \label{fig:halting-decider-from-M-underfit}
\end{figure}

We show that $\LU$ is undecidable using a reduction from the halting problem. By way of contradiction, assume that $\LU$ is decidable. There then exists a Turing machine, denoted $M_{\text{underfit}}$, which halts for all inputs of the form $\langle \mathcal{A}\rangle$,$d$ and determines whether $\mathcal{A}$ will always produce models that underfit $d$, that is, which do not transfer enough information from the dataset to the model to capture the relationship of the training data. In formal terms, a model $g^i$ at iteration $i$ underfits dataset $d$ whenever $C_{\mathcal{A}}(g^i,d) < C_d$, in accordance with Definition~\ref{def:MODEL-UNDERFIT}.

We now use $M_\text{underfit}$ to construct a decider for $L_{\text{halt}}$ with the given steps (this decider is also shown in Figure \ref{fig:halting-decider-from-M-underfit}):
    
We create an auxiliary machine called \textit{$\mathcal{A}'$ builder}, which takes as input a Turing machine encoding and input string, $\langle M \rangle, w$. $\mathcal{A}'$ builder then creates an encoded algorithm $\mathcal{A}'$ representing an iterative learning method, exporting the encoded algorithm together with a training dataset $d$. The dataset $d$ consists of a single training example with $k-1$ binary features and a single binary label, for some positive integer $k > 1$. We construct $d = \{(x_1, y_1)\}$ as follows, where $x_1$ denotes features and $y_1$ the label. The features are chosen uniformly at random, as is the label; simply put, we generate a $k$-length binary string by flipping a fair coin, and take the first $k-1$ bits as the features $x_1$, and the final bit as $y_1$. This produces one of $2^k$ possible datasets with equal probability, where $C'_d = k$ for each one (since the $k$-length binary string fully encodes the feature-label pair, without compression). Thus, $p(d) = 2^{-k} = 2^{-C'_d}$. Also note that $C_d > 0$ because $d$ is nonempty and so any encoded Turing machine that produces it must consist of one or more bits.

$\mathcal{A}'$ works as follows:
On its initial iteration, it produces a learning model that outputs a constant zero value for all labels, independent of the data. We consider this point in time $t_1$ and let $g^{t_1}$ represent this initial constant model. Thus, $\mathcal{A}$ has a pointwise information transfer of zero, $C_{\mathcal{A}}(g^{t_1},d) = 0$, since $p(g^{t_1}|d) = p(g^{t_1})$, as the algorithm will produce this initial model with probability 1, independent of $d$. Therefore, $$0 = C_{\mathcal{A}}(g^{t_1},d) < C_d,$$ and we see that at time point $t_1$ all of $\mathcal{A}$'s models produced so far (namely, the single model $g^{t_1}$) will \textbf{underfit} by Definition \ref{def:MODEL-UNDERFIT}.

After this, $\mathcal{A}'$ simulates $M$ on $w$. If $M$ halts on input $w$, $\mathcal{A}'$ then updates its internal model as follows. For any input with features equal to $x_1$ it predicts the label $y_1$, effectively memorizing the datapoint. For all other inputs, it produces the label $1 - y_1$, that is, it negates the label $y_1$. For example, if $y_1 = 1$ and $x_1 = (0,0,1,1)$, then the model will have exactly one response equal to $1$, namely for input $(0,0,1,1)$, and all other inputs will map to a response label of $0$. If $y_1 = 0$, we would have exactly one response with value of $0$, and all other inputs would map to $1$. Thus, at time point $t_2$ we will have one of $2^k$ possible models. Note that a bijective mapping exists between the model $g^{t_2}$ and the dataset $d$; changing a single bit in the dataset $d$ results in a different model. Therefore, $p(g^{t_2}|d) = 1$ (the algorithm's choice is deterministic given $d$), and 
\begin{align*}
p(g^{t_2}) &= \sum_{d'} p(g^{t_2}|d')p(d') \\
    &= 1 \cdot p(d) + \sum_{d' \not = d} 0 \cdot p(d') \\
    &= 2^{-C'_d}
\end{align*}
by construction. Using Definition~\ref{defn:pointwise-transfer} we obtain
\begin{align*}
C_{\mathcal{A}}(g^{t_2},d) 
&= \log_2 p(g^{t_2}|d)/p(g^{t_2})\\
&= \log_2 2^{C'_d}\\
&= C'_d \\
&\geq C_d
\end{align*}
and $\mathcal{A}'$ produces a model that does \textbf{not underfit} at time point $t_2$, according to Definition~\ref{def:MODEL-UNDERFIT}.

If $M$ does not halt on input $w$, then the algorithm only completes a single iteration, leaving the original underfitting model intact. Thus, $\mathcal{A}'$ will underfit $d$ at all iterations if and only if machine $M$ does not halt on input $w$.

Under the assumption that $M_{\text{underfit}}$ exists, we can pass the outputs of $\mathcal{A}'$ builder to this machine and ask $M_{\text{underfit}}$ if $\mathcal{A}'$ underfits $d$ at all iterations. The way it answers will tell us whether $M$ halts on $w$, as it will always underfit only in the case that $M$ does not halt on input $w$. The outputs from $M_{\text{underfit}}$ are then swapped and routed to the output of machine $M_{\text{halt}}$, creating a decider for $L_{\text{halt}}$, which is a contradiction. Having reached a contradiction, our initial assumption that $M_{\text{underfit}}$ exists cannot hold, and $\LU$ is therefore undecidable.
\end{proof}

\section{Discussion}
We have proven that, under a rigorous and reasonable information-theoretic definition of underfitting, the problem of determining whether an arbitrary learning algorithm will eventually fit a dataset given enough training time is formally undecidable. While a perfect algorithm cannot exist, this does not rule out underfitting detection in special cases, such as for fixed-capacity algorithms. However, our theorem guarantees no underfitting predictor can be universally applicable: it will either apply only to a subset of algorithms, will sometimes produce incorrect results, or will fail to terminate when applied to some algorithm and dataset pairs.

Opportunities for future work remain.
Just as Bashir et al.\ established bounds for algorithm capacity and distributional algorithm capacity \cite{bashir-2020-ITP}, it should be possible to create bounds using time-indexed capacity. Creating such bounds may allow for new insights on when algorithms might underfit. Research into the special case of fixed-capacity learning algorithms seems especially promising.
Another direction to explore is whether the degree to which an algorithm underfits, captured by $\E_{\D}[C_D]-C^i_{\mathcal{A}}$, can be used to bound a learning method's generalization error.

\section{Conclusion}

Underfitting remains a problem when training iterative algorithms. In trying to understand and describe this phenomenon, we may be tempted to create a list of criteria to predict exactly when an algorithm will underfit.
While being able to determine with certainty that an arbitrary algorithm will underfit would be useful, we show that in general this cannot be done, for if it could, we would also be able to decide the halting problem.

Although it is impossible to always determine whether an algorithm will underfit, this does not rule out probabilistic bounds on the likelihood of underfitting nor exact determination for specific classes of learning algorithms. Future work may include bounding the probability of underfitting given an algorithm and a dataset.
Investigating these questions may also lead to a better understanding of why common solutions for underfitting work, by linking these strategies to information-theoretic notions of algorithm capacity and mutual information.

\section*{Author Contributions}

S. Sehra is responsible for the initial formulation of the paper and proof and for the key insight that the underfitting problem may be formally undecidable. D. Flores contributed to the Related Work section, much of the Introduction and Conclusion, and the Discussion section. G. Monta\~{n}ez contributed Definition~\ref{def:MODEL-UNDERFIT} and formally proved Theorem~\ref{thm:UNDECIDABLE-UNDERFITTING}, with support in figure creation and layout from Sehra and Flores. Monta\~{n}ez contributed a majority of the text from the preamble of  Section~\ref{sec:UNDECIDABILITY}. All authors contributed to the prose of the manuscript, and all authors participated in editing the final manuscript, sharing responsibility for the conclusions and content contained herein.
\vfill
\bibliographystyle{IEEEtran}
\bibliography{references}

\end{document}